\documentclass{article}

\PassOptionsToPackage{numbers, compress}{natbib}

\usepackage[preprint]{neurips_2024}

\usepackage[utf8]{inputenc} 
\usepackage[T1]{fontenc}    
\usepackage[colorlinks=true,citecolor=brown]{hyperref}       
\usepackage{url}            
\usepackage{booktabs}       
\usepackage{amsfonts}       
\usepackage{nicefrac}       
\usepackage{microtype}      
\usepackage[table,xcdraw, dvipsnames]{xcolor}

\usepackage{amsmath,amsfonts,bm}

\def\eqref#1{equation~\ref{#1}}

\def\1{\bm{1}}

\DeclareMathAlphabet{\mathsfit}{\encodingdefault}{\sfdefault}{m}{sl}
\SetMathAlphabet{\mathsfit}{bold}{\encodingdefault}{\sfdefault}{bx}{n}

\usepackage{amsthm}
\usepackage{enumitem}
\usepackage{graphicx}
\usepackage{multirow}
\usepackage{array}
\usepackage{arydshln}
\usepackage{algorithm}
\usepackage{algpseudocode}
\usepackage{subfig} 
\usepackage{comment}
\usepackage{wrapfig}  
\usepackage{parskip}
\usepackage{mdframed}
\usepackage{courier} 
\usepackage{listings}
\usepackage{tcolorbox}

\definecolor{customblue}{HTML}{4A86E8} 
\newmdenv[linecolor=customblue!10, backgroundcolor=customblue!7]{lightbluebox}
\newmdenv[linecolor=customblue!10, backgroundcolor=lightgray!12]{graybox}
\newcommand{\newtt}[1]{{\fontfamily{pcr}\selectfont #1}}

\newtheorem{assumption}{Assumption}
\newtheorem{theorem}{Theorem}
\newtheorem{definition}{Definition}
\newtheorem{corollary}{Corollary}

\title{Efficient Model-agnostic Alignment \\ via Bayesian Persuasion}

\author{%
  Fengshuo Bai$^{1,2,3}$
  Mingzhi Wang$^{2,\dag}$\thanks{$^{\dag}$ Equal Contribution. $^{\ddag}$Correspondence to Yaodong Yang$<$\textit{yaodong.yang@pku.edu.cn}$>$ and Ying Wen $<$\textit{ying.wen@sjtu.edu.cn}$>$.} \,
  Zhaowei Zhang$^{2,3,\dag}$ 
  Boyuan Chen$^{2,\dag}$ 
  Yinda Xu$^{1}$\\
  \textbf{Ying Wen$^{1,\ddag}$ \quad Yaodong Yang$^{2,\ddag}$} \\
  $^{1}$Shanghai Jiao Tong University \\
  $^{2}$Institute for Artificial Intelligence, Peking University \\
  $^{3}$National Key Laboratory of General Artificial Intelligence, BIGAI\\
}

\begin{document}
\maketitle

\begin{abstract}
 With recent advancements in large language models (LLMs), alignment has emerged as an effective technique for keeping LLMs consensus with human intent. Current methods primarily involve direct training through Supervised Fine-tuning (SFT) or Reinforcement Learning from Human Feedback (RLHF), both of which require substantial computational resources and extensive ground truth data. This paper explores an efficient method for aligning black-box large models using smaller models, introducing a model-agnostic and lightweight Bayesian Persuasion Alignment framework. We formalize this problem as an optimization of the signaling strategy from the small model's perspective. In the persuasion process, the small model (Advisor) observes the information item (i.e., state) and persuades large models (Receiver) to elicit improved responses. The Receiver then generates a response based on the input, the signal from the Advisor, and its updated belief about the information item. Through training using our framework, we demonstrate that the Advisor can significantly enhance the performance of various Receivers across a range of tasks. We theoretically analyze our persuasion framework and provide an upper bound on the Advisor’s regret, confirming its effectiveness in learning the optimal signaling strategy. Our Empirical results demonstrates that GPT-2 can significantly improve the performance of various models, achieving an average enhancement of 16.1\% in mathematical reasoning ability and 13.7\% in code generation. We hope our work can provide an initial step toward rethinking the alignment framework from the Bayesian Persuasion perspective.
\end{abstract}

\section{Introduction}
Recent years have witnessed increased attention and effort in aligning large language models (LLMs) with human intentions and values \citep{openai2023gpt4,ji2023ai}. This alignment is facilitated by providing reliable supervision through demonstrations \citep{brown2020language,taori2023stanford}, reward signals \citep{ouyang2022training}, preferences \citep{christiano2017deep,rafailov2023direct} or critiques \citep{saunders2022self,bai2022constitutional}, and by employing methods such as supervised learning (\textit{e.g.}, Supervised Fine-tuning, SFT) or reinforcement learning (\textit{e.g.}, Reinforcement Learning from Human Feedback, RLHF) \citep{ouyang2022training}. 

However, these methods, including RLHF, require multiple models and direct training of large models, which significantly increases computational demands \citep{rafailov2023direct}. Moreover, Fine-tuning cannot be applied to closed-source models, complicating output control for alignment with human intents. Additionally, current alignment methods, like SFT and RLHF, face limitations when human evaluators lack expertise in complex tasks \citep{bowman2022measuring,sharma2023towards}. These challenges highlight the need for efficient, scalable alignment strategies for both open-source and closed-source models.
Therefore, our work aims to address the above challenges by answering the following question: 
\begin{center}
   \textbf{\textit{Can we use a smaller model to influence the behaviors of larger models, thereby enabling alignment and enhancing performance with little human supervision or feedback?}}  
\end{center} 
Inspired by Bayesian persuasion \citep{kamenica2011bayesian}, we model the alignment problem as an information design process, involving a protocol between a small model and a large model. Instead of training multiple models as in techniques like RLHF, we only employ a small model with minimal supervision to learn a signaling strategy that influences the behaviors of fixed large models. In this setup, large models are treated as a black box. This modeling approach significantly reduces the demand for computational resources by delegating alignment tasks to smaller models, leading to a more lightweight and efficient framework.
This makes it inherently suitable for a broader range of alignment scenarios, accommodating tasks of varying difficulty and larger models.

\begin{figure}
    \centering
    \includegraphics[width=0.95\linewidth]{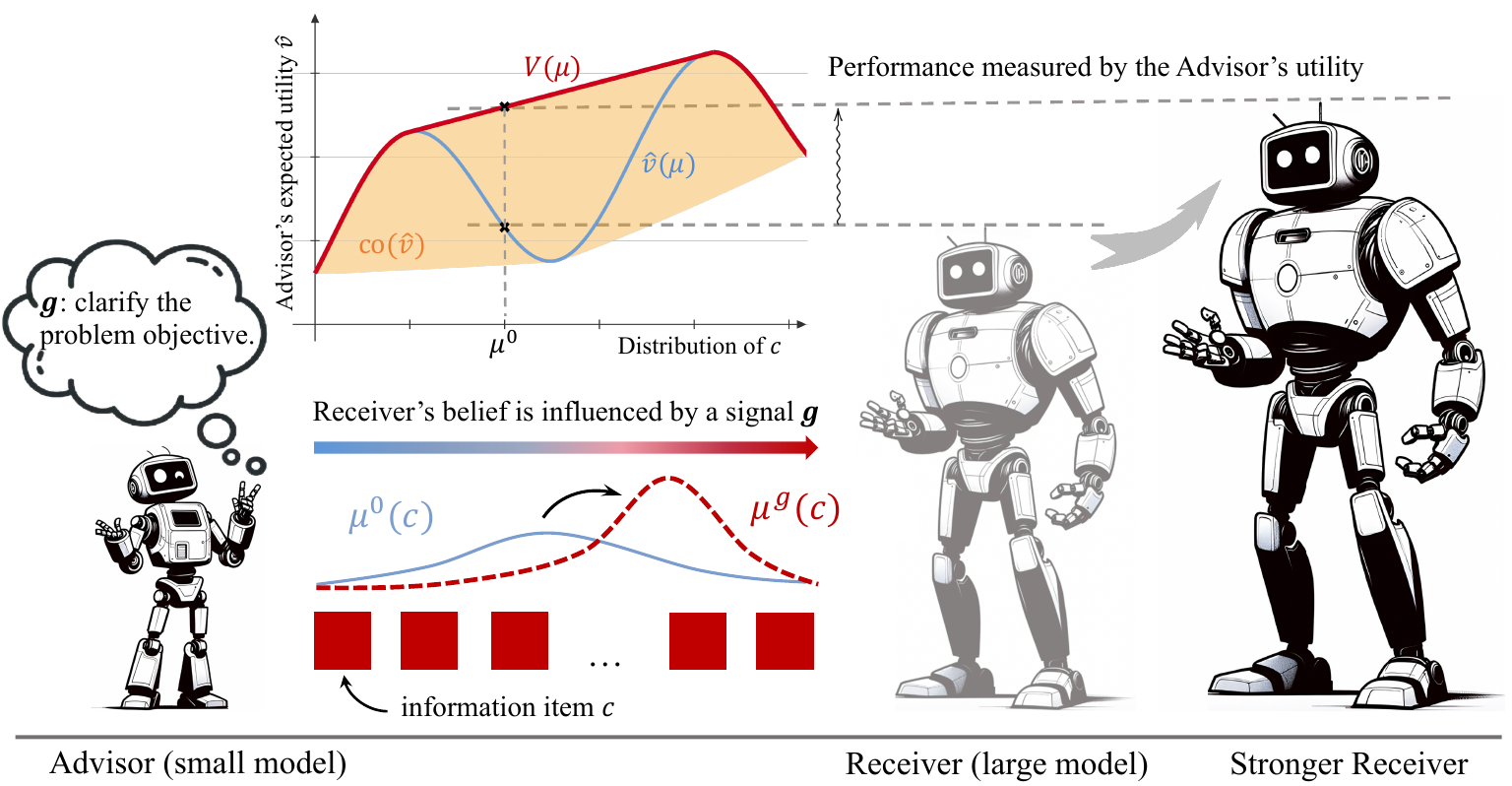}
    \caption{An illustration of our persuasion framework. The Receiver observes a signal $g$ sent by the Advisor and updates its belief from the \textcolor[HTML]{4A90E2}{prior distribution $\mu^0$} to a \textcolor[HTML]{D0021B}{posterior distribution $\mu^g$}. The axes depict the Advisor's expected utility $\hat{v}(\mu)$ across various information distributions $\mu$. In this context, $co(\hat{v})$ denotes the convex hull of $\hat{v}$, while $V$ represents the concave closure of $\hat{v}$. Here, $V(\cdot)$ is the largest expected utility Advisor can achieve with any signal. From the Advisor's perspective, the Receiver's performance is enhanced following persuasion.
    }
    \label{fig:main-paradigm}
\end{figure}

In this work, we introduce a novel framework termed Bayesian Persuasion Alignment, as illustrated in Figure \ref{fig:main-paradigm}. We frame alignment as a Bayesian Persuasion (BP) problem, wherein a smaller model serves as the Advisor and a larger model acts as the Receiver. In this setup, the Advisor generates a signal that is sent to the Receiver. Upon receiving this signal, the Receiver updates its beliefs and produces a response. Our core insight is that a smaller model trained on supervision for optimal signaling strategy can effectively \textit{persuade} larger models, thereby improving the quality of their responses, \textit{i.e.}, their outputs. 
This mechanism provides several advantages:
(1) From the perspective of information design, the well-defined signaling strategy of the Advisor ensures an increase in the Advisor's utility without decreasing the Receiver's utility \citep{kamenica2011bayesian}. 
(2) The Advisor manipulates the Receiver's belief to enhance performance, significantly reducing the need for training resources while ensuring alignment performance, making it an effective and parameter-efficient alignment strategy.
(3) Moreover, the learned signaling strategy can be applied to different Receivers, guaranteeing a model-agnostic nature and making it easier to generalize to harder tasks.

To the best of our knowledge, BP Alignment is the first integration of Bayesian persuasion with the alignment framework. Our main contributions can be summarized in three folds: 
\textbf{First}, we introduce a parameter-efficient and model-agnostic alignment framework that trains a smaller model to enhance the performance of various larger models.  
\textbf{Second}, we demonstrate that our persuasion framework significantly improves the performance of various large models on mathematical problem-solving and code-generation tasks. Specifically, the Advisor (Phi-2) enables significant enhancements, with an average improvement of 22.5\% on GSM8K~\citep{abs-2110-14168}, 39.0\% on MATH~\citep{NEURIPS_BENCHMARKS2021_be83ab3e}, and a 24.7\% increase on HumanEval~\citep{abs-2107-03374}.
\textbf{Lastly}, we theoretically analyze our framework and provide an upper bound on the Advisor's regret, indicating its effectiveness in learning the optimal signaling strategy.

\section{Related Work}
In this section, we will review existing works on Scalable Oversight, AI Persuasion, and Eliciting Latent Knowledge (ELK).

\textbf{Scalable Oversight.}
As models begin to achieve broadly human-level performance and take on more complex tasks that are difficult for humans to understand, providing continual, reliable feedback and ensuring that the models' behaviors align with human intents becomes challenging. This naturally raises the crucial issue of scalable oversight: \textit{how can we provide supervisory signals to more powerful AI systems and ensure they are aligned with human intents?} \citep{superalignment,amodei2016concrete,ji2023ai} Unlike current methods that focus on enhancing the capabilities of weak supervisors \citep{christiano2018supervising,bowman2022measuring,lee2023rlaif}, our framework addresses this challenge by transforming weak supervisors into persuaders and identifying the optimal signal-sending strategy to effectively influence the behaviors of stronger models.

Our work also differs from weak-to-strong generalization \citep{burns2023weak} and similar alignment methods \citep{ji2024aligner,liu2024co}. These methods face a trade-off: the strong model may either mimic the weak model, reducing performance or use its reasoning abilities to improve \citep{burns2023weak}. Additionally, they often rely heavily on ground truth labels. In contrast, our approach uses small models as Advisors to elicit the capabilities of stronger models without adding noisy labels. Guided by the information design principle, our method scales to various stronger models and challenging tasks, minimizing the need for ground truth.
We evaluate our method on various mathematical problem-solving and code-generation tasks using only GPT-2 as the Advisor. It achieves significant advancements, with an average improvement of 9.5\% on GSM8K, 22.6\% on MATH, and 13.7\% on HumanEval across a range of strong models.

\textbf{AI Persuasion.} Persuasion is a dynamic game process in which one player (sender) influences the beliefs or actions of another player (receiver) by providing informative signals, thereby affecting the outcomes for both players. AI persuasion can be categorized into two types based on the target: (1) AI persuading humans to change their original viewpoints (\textit{i.e.}, Captology) \citep{fogg2002persuasive, wang2019persuasion, matz2024potential, durmus2024persuasion}, and (2) employing persuasive signals to change the behavior of AI systems. While most existing research has concentrated on the former, studies on the latter remain relatively nascent. In this paper, we primarily explore the latter category.

\citet{zeng2024johnny} conducted a comprehensive review of decades of social science research and proposed a taxonomy to automatically generate persuasive adversarial prompts that induce unsafe behaviors in LLMs. However, their method lacks a formal definition or analysis of persuasive behavior and its impact. Bayesian persuasion \citep{kamenica2011bayesian, kamenica2019bayesian} is a symmetric information framework that utilizes to influence beliefs by strategically sharing information aligned with motivations, aiding decision-making tasks \citep{gan2022bayesian}. While its application to language models remains unexplored, \cite{zhang2024incentive} suggested it could align AI during deployment by tailoring information based on different scenarios. To the best of our knowledge, we are the first to use dialogue to apply Bayesian persuasion to LLMs and enhance their capabilities.

\textbf{Eliciting Latent Knowledge.} \citet{elk_intro,elk_summary} introduced a theoretical framework termed Eliciting Latent Knowledge (ELK), designed to extract latent knowledge from models to assess whether AI systems align with human intents. This framework includes aspects such as honesty analysis \citep{evans2021truthful}, knowledge elicitation \citep{burns2022discovering,pacchiardi2023catch}, and general task knowledge acquisition \citep{burns2023weak}. Our framework utilizes advisor persuasion to elicit strong models' latent knowledge for solving difficult tasks and has the potential for honesty analysis.

Typical ELK methods struggle to train models to report true beliefs rather than just aligning with human preferences. This issue arises because both strategies yield identical training losses, as they produce the same answers to training inputs \citep{elk_summary}. As a result, models tend to align with human expectations instead of reporting their own beliefs. Our method addresses this by decoupling the training objective into a reporting objective and a persuasion objective, focusing on optimal signaling rather than human-evaluated ground truth.

\section{Bayesian Persuasion Alignment}
In this section, we formally introduce the proposed persuasion framework. We begin by establishing the notations and outlining the persuasion protocol, followed by defining the overall objective and conducting a theoretical analysis of regret.

\subsection{Protocol and Notations}
We introduce the persuasion protocol wherein an Advisor (small model) persuades a Receiver (large model) to improve its response to a given input. For each input $x$, a finite set ${\mathcal{C}}_x$ contains all associated information items (i.e., state) with input $x$. The Receiver's utility function $u(x, c, y)$ is continuous and dependent on its response $y \in {\mathcal{Y}}$ to the input $x \in {\mathcal{X}}$ and the associated information item $c \in {\mathcal{C}}_x$. Similarly, the Advisor has a continuous utility function $v(x, c, y)$, which is contingent on the Receiver's response, input, and associated information item. Importantly, in our settings, the Advisor lacks knowledge of the Receiver’s utility function. For each input $x$, the Advisor and Receiver start with a shared prior $\mu_x^0 \in \text{int}(\Delta({\mathcal{C}}_x))$\footnote{$\text{int}(X)$ denotes the interior of the set $X$ and $\Delta(X)$ represents the set of all probability distributions on $X$}. The signaling strategy $\pi$ is defined by a finite realization space ${\mathcal{G}}$ and a family of distributions ${\pi_x(\cdot|c)}, {c\in{\mathcal{C}}_x}$ over ${\mathcal{G}}$. This strategy is implemented through a neural network with parameters $\theta$. The Advisor sends a signal, and the Receiver observes the chosen signal realization $g \in {\mathcal{G}}$ (with $|{\mathcal{G}}| < \infty$). 

The game timing in persuasion is as follows: The Advisor commits to a signaling strategy $\pi_\theta$ and announces it to the Receiver. For a given input $x$, the Advisor observes an information item sampled from $\mu_x^0$ and sends a signal $g$ to the Receiver. Upon receiving this signal, the Receiver updates its belief about information item in ${\mathcal{C}}_x$, forming posterior distribution $\mu_x^g$ via Bayes's rule. The Receiver then chooses a response from the response set, which is defined by
\begin{equation}\label{eq: receiver_op_y}
y^*(\mu_x^g) = \arg\max\limits_{y \in {\mathcal{Y}}}\mathop{\mathbb{E}}\limits_{c\sim \mu_x^g}\Big[u(x,c,y)\Big].
\end{equation}

The solution of the game is the problem of optimal signaling strategy design from the Advisor’s point of view. Taking the Receiver’s response as given, the Advisor selects a signaling strategy $\pi_\theta$ that maximizes its expected utility. Since the responses are generated only by the Receiver, the Advisor cannot directly influence the information set. Instead, the Advisor can leverage its informational advantage concerning the information item to influence the Receiver indirectly by way of signaling, thereby persuading the Receiver to generate improved responses.

\subsection{Signaling Strategy and Belief Update}
A signaling strategy of the Advisor generates a distribution over ${\mathcal{G}}$, which is the signal realization space. Formally, the signaling strategy $\pi_\theta$ comprises a function $\pi_x: {\mathcal{C}} \rightarrow \Delta({\mathcal{G}})$ for each input $x \in {\mathcal{X}}$. Upon observing an information item $c$, the Advisor sends a signal sampled from $\pi_x(c)$, where the input is $x$, and $\pi_x(g, c)$ denotes the probability of $g \in {\mathcal{G}}$ within this distribution.

The signal space ${\mathcal{G}}$ is broadly construed, including some uninformative signaling strategies. For instance, the Advisor may send the same signal regardless of the information item $c$, such that $\pi_x(c) = \pi_x(c^\prime)$ for all $c, c^\prime \in {\mathcal{C}}_x$. Without loss of generality, we assume that signals in ${\mathcal{G}}$ are perceived as distinct by the Receiver.

Upon receiving a signal $g$, the Receiver updates its posterior belief regarding the information items. The conditional probability of the information item being $c$ is defined as:
\begin{equation}
     \textrm{Pr}(c|g,\pi_x) = 
    \frac{\mu_x(c) \pi_x(g|c)}
    {\sum_{c^\prime \in {\mathcal{C}}_x} \mu_x(c^\prime) \pi_x(g| c^\prime) }.
\label{eq:belief_update}
\end{equation}

The derivation of the Receiver's posterior belief also depends on its knowledge of the signaling strategy $\pi_\theta$. In accordance with the Bayesian persuasion framework, the Advisor commits to a signaling strategy $\pi_\theta$ at the beginning of the process and announces it to the Receiver.

\subsection{Signaling Strategy Optimization}
From the Advisor's perspective, the objective is to identify a signaling strategy $\pi_\theta$ that maximizes the Advisor's expected utility, thereby inducing superior responses from the Receiver.
Accordingly, the signaling strategy is optimized by minimizing the following loss function:
\begin{equation}
    {\mathcal{L}}(\theta) = -\mathop{\mathbb{E}}_{x\in{\mathcal{X}}}
    \Big[ \sum_{c\sim{\mathcal{C}}_x}{\textrm{Pr}(c|g,\pi_x)} \ v(x,c,y) \Big],
\end{equation}
where $y$ is the Receiver's response to input $x$ as determined by equation (\ref{eq: receiver_op_y}).

\subsection{Regret Analysis}
Although the Bayesian Persuasion Alignment framework we propose is practically appealing, a pivotal theoretical question arises: \textit{How can we ensure that this framework robustly learns the optimal signaling strategy over time?} Specifically, it is crucial to demonstrate that the Advisor can gradually find the most persuasive signaling strategy through its interactions with the Receiver. This convergence guarantee is essential for our framework. To address this question, we draw inspiration from the online Bayesian persuasion setting \citep{castiglioni2020online, bernasconi2023optimal} and analyze the performance of our algorithm from an online learning perspective. We introduce the concept of regret, which quantifies the utility difference between the algorithm’s performance and the optimal strategy over a certain period. Demonstrating that the algorithm’s regret grows sublinearly with time would imply that it can progressively converge to the optimal strategy.

Without loss of generality, we focus on signaling schemes that are both \emph{direct} and \emph{persuasive} \citep{arieli2019private}, according to the revelation principle. In a \emph{direct} signaling scheme, the signals directly correspond to response recommendations for the Receiver. Moreover, a signaling scheme is considered \emph{persuasive} if it incentives the Receiver to follow the response recommendations provided by the Advisor. Let ${\mathcal{P}}$ denote the set of \emph{direct} and \emph{persuasive} signaling schemes, where each element $\phi \in \mathcal{P}$ is a mapping $\phi: \mathcal{C} \to \Delta(\mathcal{Y})$. To simplify the notation, we omit $x$ in subsequent analysis. With the definition of the set of persuasive signaling schemes $\mathcal{P}$, the Advisor's expected utility under a signaling scheme $\phi \in \mathcal{P}$ can be expressed as follows:
\begin{equation}
   v(\phi) := \sum_{c \in \mathcal{C}} \sum_{y \in \mathcal{Y}} \mu_{c} \phi_{c}(y) v(c, y),
\end{equation}
where $\mu_c$ represents the prior probability of information item $c$, $\phi_c(y)$ denotes the probability that the signaling scheme $\phi$ recommends response $y$ under information item $c$, and $v(c, y)$ is the Advisor's utility when the Receiver takes response $y$ under information item $c$.

Next, we introduce a linear mapping $f$ that maps each signaling scheme $\phi \in \mathcal{P}$ to a point in the $\mathbb{R}^{|\mathcal{Y}|}$ space. Specifically, for each $\phi \in \mathcal{P}$, we define
\begin{equation}
  f(\phi) := \left[-v(\phi, y)\right]_{y \in \mathcal{Y}},
\end{equation}
where $v(\phi, y) = \sum_{c \in \mathcal{C}} \mu_{c} \phi_{c}(y) v(c, y)$ represents the Advisor's expected utility when the Receiver takes response $y$ under signaling scheme $\phi$.

Intuitively, the mapping $f$ represents each signaling scheme as a $|\mathcal{Y}|$-dimensional vector, where each component $-v(\phi, y)$ represents the negative of the Advisor's expected utility for response $y$. This representation embeds the signaling schemes into a Euclidean space that directly corresponds to the Receiver's response space. Furthermore, we examine the convex hull of the graph of $f$, denoted as $\mathrm{co}\,f(\mathcal{P})$. Each point within $\mathrm{co}\,f(\mathcal{P})$ corresponds to a convex combination of signaling schemes. 

Formally, the Advisor's regret at round $T$ is defined as:
\begin{equation}
    R_T := \max_{\phi \in {\mathcal{P}}} \sum_{t=1}^T v(\phi) - \sum_{t=1}^T v(\phi_t).
\end{equation}

To analyze the regret bound of our persuasion framework, we present the theoretical version of our algorithm in Algorithm \ref{algo:online_bp}.
\begin{algorithm}[H]
\caption{Theoretical Persuasion Algorithm}
\label{algo:online_bp}
\begin{algorithmic}[1]
\Require Set of information items ${\mathcal{C}}$, set of responses ${\mathcal{Y}}$, prior distribution $\mu^{0}$, Advisor's utility function $v$, regret minimizer $\mathcal{R}$ for the set $\mathrm{co}\,f({\mathcal{P}})$
\For{$t = 1, \ldots, T$}
    \State $\mathrm{co}\,f({\mathcal{P}}) \ni z_t \gets \mathcal{R}.\text{RECOMMEND()}$
    \State $\{(\phi_{t}^{(i)}, \lambda_{t}^{(i)})\}_{i \in [m+1]} \gets \text{DECOMPOSE}(z_t, f({\mathcal{P}}))$ \Comment{Caratheodory's Theorem}
    \State Sample $i_t \in [m+1]$ with probabilities $\lambda_{t}^{(1)}, \ldots, \lambda_{t}^{(m+1)}$
    \State Let $\phi_t \gets \phi_{t}^{(i_t)}$
    \State Observe information item $c_t \sim \mu^{0}$
    \State Select and play action $y_t \sim \phi_t(\cdot|c_t)$
    \State $\mathcal{R}.\text{OBSERVE}(v(c_t, y_t))$
\EndFor
\end{algorithmic}
\end{algorithm}

The key idea behind this algorithm is to maintain a regret minimizer $\mathcal{R}$ over the possible signaling strategies, represented by the convex hull of the set of posterior distributions $\mathrm{co}\,f({\mathcal{P}})$. At each round $t$, the algorithm obtains a recommended strategy $z_t$ from $\mathcal{R}$ and decomposes it into a convex combination of extreme points $\{(\phi_{t}^{(i)}, \lambda_{t}^{(i)})\}_{i \in [m+1]}$ using Caratheodory's Theorem \citep{cook1972caratheodory}. The algorithm then samples an index $i_t$ according to the weights $\lambda_{t}^{(i)}$ and plays the corresponding signaling strategy $\phi_t = \phi_{t}^{(i_t)}$. Upon observing the realized information item $c_t$ and the Advisor's utility $v(c_t, y_t)$ for the chosen reponse $y_t$, the algorithm updates the regret minimizer with this feedback.

Under this theoretical algorithm, we can derive the following regret bound.
\begin{theorem}
\label{tm:regret_bound}
Algorithm \ref{algo:online_bp} guarantees regret $R_T = O(m^{3/2} \sqrt{T \log T})$, where $m = |{\mathcal{Y}}|$ is the number of receiver's reponses.
\end{theorem}
The regret bound presented in Theorem \ref{tm:regret_bound} demonstrates that our algorithm achieves sublinear regret over the time horizon $T$, with a dependence on the size of Receiver's response space $m$. The output space of LLMs is theoretically infinite, as they can generate text of arbitrary length. However, each response's length is practically limited. Additionally, responses with same semantics are considered equivalent given a specific input. Therefore, the Receiver's response space can be regarded as finite, aligning well with the assumptions in our theoretical analysis. Although there are differences between the theoretical algorithm and its practical implementation, the core principle of learning the optimal signaling strategy through interaction remains consistent. This consistency provides a theoretical guarantee for the algorithm's performance and demonstrates its effectiveness in learning the optimal signaling strategy over time.

\section{Experiments}
In this section, we evaluate the effectiveness of our persuasion framework on mathematical problems and code generation. Our evaluation aims to address the following key questions:
\vspace{-0.6em}
\begin{itemize}[leftmargin=*, label={}]
\item \textbf{(1)} Can our framework enhance the Receiver's performance in various tasks? (Section ~\ref{subsubsec:perf_persuasion})
\item \textbf{(2)} Can our framework find a non-trivial signaling strategy? (Section \ref{subsubsec:info_structure})
\item \textbf{(3)} How about the efficiency of the proposed framework? (Section \ref{subsubsec:efficiency})
\end{itemize}
\vspace{-0.3em}
Furthermore, we investigate the generalization of the signaling strategy across different Receivers (Section~\ref{subsubsec:perf_persuasion}), across varying difficulties (Section~\ref{subsubsec:easy2hard}), and for various tasks (Appendix~\ref{app:general_on_task}). Details on experiments are provided in Appendix ~\ref{app:train_detail}.

{\large
\setlength{\extrarowheight}{1.1pt}
\begin{table}[]
\centering
\caption{\textbf{Performance of various Receivers under persuasion.} 
We report the accuracy on GSM8K and MATH, and the pass@1 score on HumanEval across four information structures. "Posterior Information" refers to sampling the information item from the posterior distribution, influenced by the Advisor. The Advisor for math tasks differs from that for code generation tasks. Arrows indicate performance improvements relative to the prior distribution.
}

\label{tab:math_results}
\resizebox{\textwidth}{!}{
\begin{tabular}{@{}clccccc@{}}
\toprule
\multirow{2}{*}{Task} &
  \multicolumn{1}{l}{\multirow{2}{*}{Receiver}} &
  \multicolumn{1}{c}{\multirow{2}{*}{\parbox{2cm}{\centering No \\Information}}} &
  \multicolumn{1}{c}{\multirow{2}{*}{\parbox{2cm}{\centering All \\Information}}} &
  \multicolumn{1}{c}{\multirow{2}{*}{\parbox{2cm}{\centering Prior \\Information}}} &
  \multicolumn{2}{c}{\textbf{Posterior Information}} \\ \cmidrule(l){6-7}\cmidrule(l){6-7}
 &
  \multicolumn{1}{c}{} &
  \multicolumn{1}{c}{} &
  \multicolumn{1}{c}{} &
  \multicolumn{1}{c}{} &
  \multicolumn{1}{c}{Advisor (GPT-2)} &
  \multicolumn{1}{c}{Advisor (Phi-2)} \\ \midrule
\multirow{12}{*}{\parbox{2cm}{\centering GSM8K \\ (8-shot)}}
        & Phi-2      &56.0  &41.0  &56.8  &59.1  &62.1  \\
        & Mistral-7B &34.3  &48.0  &45.7  &50.4  &53.8  \\
        & Llama2-7B  &15.1  &36.6  &27.2  &34.5  &45.6  \\
    & Llama2-7B-Chat &21.8  &31.8  &37.3  &40.0  &50.0  \\
        & Llama2-13B &25.2  &38.9  &36.2  &38.9  &45.9  \\
    & Llama2-13B-Chat&33.9  &37.3  &36.1  &37.9  &39.2  \\
        & Llama3-8B  &47.6  &54.0  &53.7  &56.0  &62.6  \\
&Llama3-8B-Instruct  &73.5  &72.2  &72.3  &74.5  &75.4  \\
        & Vicuna-7B  &14.9  &19.9  &29.9  &35.1  &43.2  \\
        & Vicuna-13B &23.0  &24.8  &35.0  &43.9  &49.2  \\
        & Vicuna-33B &43.2  &44.1  &47.8  &53.1  &58.5  \\ \cdashline{2-7}\cdashline{2-7}
& Average (accuracy) &35.3  &40.8  &43.5  &47.6 (\textbf{9.5\%} $\uparrow$)  &53.2 (\textbf{22.5\%} $\uparrow$)  \\ \midrule
\multirow{12}{*}{\parbox{2cm}{\centering MATH \\ (4-shot)}}
        & Phi-2      &10.1  &11.6  &11.5  &13.9  &15.4  \\
        & Mistral-7B &6.4   &10.3  &7.9   &9.3   &10.8  \\
        & Llama2-7B  &4.1   &9.5   &6.3   &8.6   &10.3  \\
    & Llama2-7B-Chat &4.6   &7.8   &6.0   &8.0   &10.4  \\
        & Llama2-13B  &4.5   &9.7   &7.7   &9.6   &11.4  \\
    & Llama2-13B-Chat &5.2   &9.8   &7.3   &9.2   &10.5  \\
        & Llama3-8B  &11.0  &16.1  &12.8  &15.9  &16.0  \\
&Llama3-8B-Instruct  &18.1  &18.6  &18.1  &18.9  &19.7  \\
        & Vicuna-7B  &3.8   &10.1  &6.7   &8.8   &10.5  \\
        & Vicuna-13B &3.8   &11.1  &6.7   &9.5   &11.0  \\
        & Vicuna-33B &6.8   &13.1  &9.3   &11.2  &13.4  \\ \cdashline{2-7}\cdashline{2-7}
& Average (accuracy) &7.1   &11.6  &9.1   &11.2 (\textbf{22.6\%} $\uparrow$)  &12.7 (\textbf{39.0\%} $\uparrow$)\\ \midrule
\multirow{12}{*}{\parbox{2cm}{\centering HumanEval \\ (0-shot)}}
        & Phi-2      &45.7  &35.1  &39.6  &45.7  &49.4  \\
        & Mistral-7B &28.3  &32.4  &31.2  &33.2  &35.4  \\
        & Llama2-7B  &12.2  &16.2  &14.4  &16.2  &18.3  \\
    & Llama2-7B-Chat &13.4  &16.3  &12.6  &15.6  &22.6  \\
        & Llama2-13B &17.1  &17.7  &16.5  &19.5  &21.3  \\
    & Llama2-13B-Chat&19.5  &17.2  &16.1  &17.4  &22.0  \\
        & Llama3-8B  &27.4  &23.8  &24.6  &31.7  &37.2  \\
& Llama3-8B-Instruct &56.7  &48.1  &53.2  &56.3  &59.5  \\
    & CodeLlama-7B   &31.1  &30.5  &28.9  &31.2  &32.9   \\
    & CodeLlama-13B  &36.5  &39.0  &35.4  &40.8  &40.2  \\
    & CodeLlama-34B  &48.7  &45.1  &41.8  &49.7  &53.2  \\ \cdashline{2-7}\cdashline{2-7}
   & Average (pass@1)&30.6  &29.2  &28.6  &32.5  (\textbf{13.7\%} $\uparrow$) &35.6 (\textbf{24.7\%} $\uparrow$) \\ \bottomrule
\end{tabular}
}
\end{table}
}
\subsection{Settings}
\textbf{Implementation Details.} We train the Advisor for math-solving tasks using the training datasets from GSM8K and MATH, and for the code generation task using the training dataset from MBPP. We construct an information set for each input using Llama3-8B-Instruct \footnote{https://github.com/meta-llama/llama3} for all datasets. Specifically, each input includes seven information items, each emphasizing different key aspects essential for problem-solving. Further details on information set construction are provided in Appendix \ref{app:prior_construction}. In practical implementation, the Advisor's utility function is defined as the logarithm of the probability of generating the correct answer, while the Receiver's utility function is $u(x, c, y) = P(y|x,c)$, naturally aligning the model's inherent mechanism with the Receiver's behavior. A detailed description and intuition for the setup of the utility function are provided in Appendix \ref{app:utility_func}.

\textbf{Datasets.} For a comprehensive evaluation of the ability, we select two kinds of tasks: math problems and code generation. 
\begin{itemize}[leftmargin=*]
\vspace{-0.6em}
\item \textbf{GSM8K} \citep{abs-2110-14168} is high-quality linguistically diverse grade school math word problems created by human problem writers, which contains 7.5k training problems and 1k test problems.
\item \textbf{MATH} \citep{NEURIPS_BENCHMARKS2021_be83ab3e} is a dataset of challenging competition mathematics problems, which is segmented into 7.5k training problems and 5k testing problems.
\item \textbf{HumanEval} \citep{abs-2107-03374} is a code evaluation benchmark consisting of 164 original programming questions, assessing language comprehension, algorithms, and basic mathematics, with some questions equivalent to simple software interview questions.
\item \textbf{MBPP} \citep{abs-2108-07732} consists of approximately 1k crowdsourced Python programming problems, covering basic programming knowledge, standard library functionalities, etc. This dataset is only used for the training of the Advisor models.
\vspace{-0.6em}
\end{itemize}

\textbf{Advisor and Receiver.}
In our persuasion framework, we employ two models: an Advisor (small model) and a Receiver (large model). For the Advisor, we select two well-known open-source small models: \textbf{GPT-2} \citep{radford2019language} (124M) and \textbf{Phi-2} \citep{javaheripi2023phi} (2.7B). To broadly investigate the generalization of the proposed method across various models, we consider several large models as Receivers: \textbf{Phi-2} \citep{javaheripi2023phi} (2.7B), \textbf{Llama-2} \citep{touvron2023llama} (7B, 13B), \textbf{Llama-3} (8B), \textbf{CodeLlama} \citep{roziere2023code} (7B, 13B, 34B), \textbf{Vicuna} \citep{vicuna2023} (7B, 13B, 33B), and \textbf{Mistral} \citep{jiang2023mistral} (7B).

\textbf{Evaluation Metrics.}
For the math problems, we determine accuracy by extracting the last number from the generated responses and comparing it directly to the ground truth. For the code generation tasks, our evaluation focuses on assessing the functional correctness of LLM-synthesized code. We use the unbiased version of the pass@1 \citep{abs-2107-03374} setting for both tasks, namely only generating one result per round. In practice, we use the tool chain-of-thought-hub\footnote{https://github.com/FranxYao/chain-of-thought-hub/tree/main} and DeepSeek-Coder\footnote{https://github.com/deepseek-ai/DeepSeek-Coder} to perform the evaluation process for math and code generation tasks, respectively.

\textbf{Evaluation Settings.}
For any given input, there is a corresponding set of information item, each item of which is related to the input. In our experiments, we examine four information structures. Given the specific input, the Receiver may observe: (1) \textit{No Information} items, (2) \textit{All Information} items, (3) an item sampled from the \textit{Prior Information} distribution, or (4) an item sampled from the \textit{Posterior Information} distribution. Naturally, the variation in information structure has an impact on the quality of the Receiver's response.

\subsection{Results}\label{subsec:results}
We evaluate the Advisor with various Receiver models to investigate the effectiveness of its signaling strategy on math problem-solving and code generation tasks. Table \ref{tab:math_results} shows the Advisor improves the performance of various models through persuasion instead of training. Additional experiments demonstrate that our persuasion framework can identify a non-trivial signaling strategy, which exhibits superior performance in terms of efficiency and generalization.

\subsubsection{Performance on Persuasion}\label{subsubsec:perf_persuasion}
To investigate the effectiveness of our persuasion framework, we conduct an experimental evaluation of the Receiver's behavior under prior information distribution and posterior information distribution. Table~\ref{tab:math_results} illustrates that Advisor can significantly improve different Receiver's performance across a variety of tasks. Comparing the Receiver's performance without additional information to that with prior information, we find that additional information enhances the Receiver's performance. From the perspective of persuasion, prior and posterior distributions share the same information set. Instead of training, the Advisor (GPT-2) can significantly enhance the performance of various models, achieving an average improvement of 9.5\% on GSM8K, 22.6\% on MATH, and 13.7\% on HumanEval. When considering the increment in the model parameters for the Advisor, a larger one (Phi-2) enables significant enhancements, with an average improvement of 22.5\% on GSM8K, 39.0\% on MATH, and a 24.7\% increase on HumanEval. One important observation can be noticed: a good signaling strategy by the Advisor can effectively persuade different Receivers.

\begin{figure}[t!]
\centering
\begin{center}
  \includegraphics[width=0.95\linewidth]{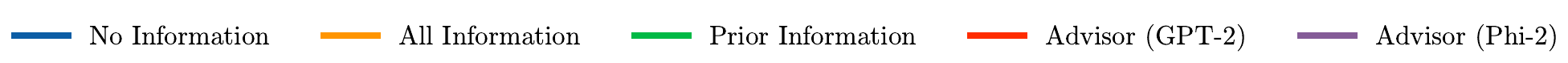}
\end{center}
\begin{tabular}{cc}
    \subfloat[MATH Level 1-3 (easy)]{\includegraphics[width=0.44\linewidth]{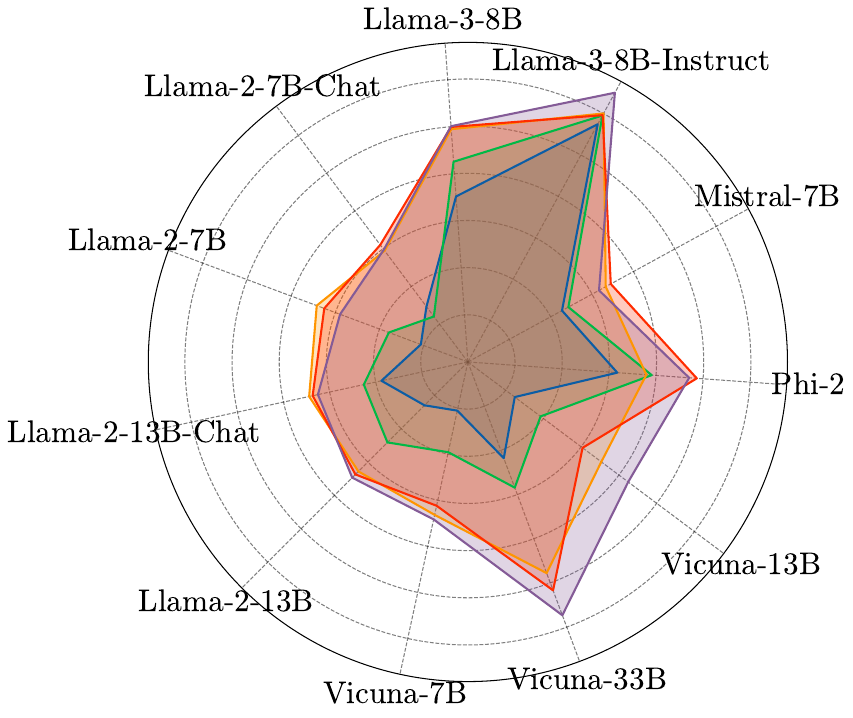}\label{fig:test1}}\hspace{2em}
    \subfloat[MATH Level 4-5 (hard)]{\includegraphics[width=0.44\linewidth]{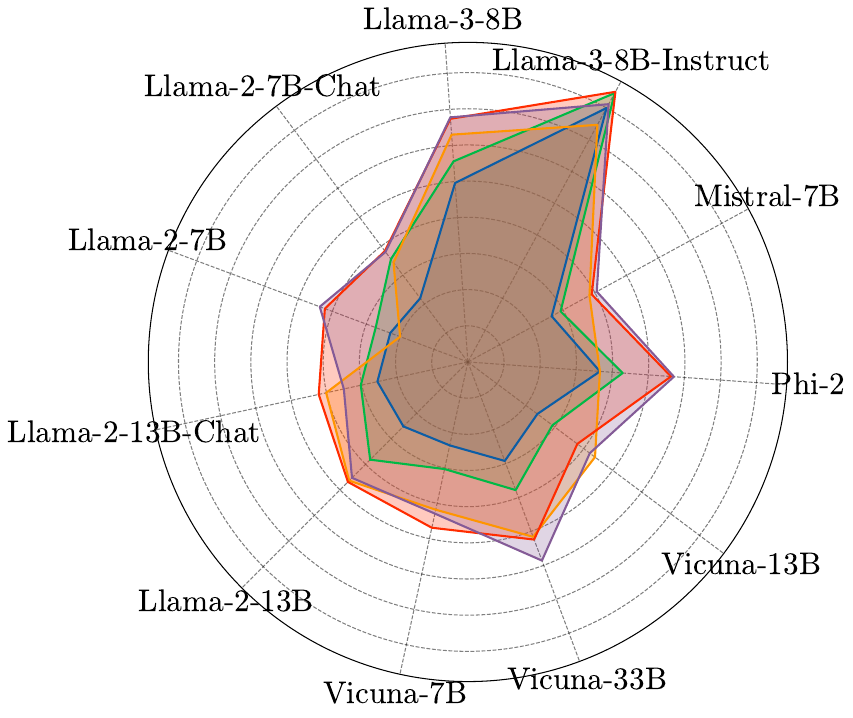}\label{fig:test2}}
\end{tabular}
\caption{\textbf{Performance of Receiver on easy and hard problems.} The Advisor (GPT-2 and Phi-2) is trained on easy problems of training set (MATH level 1-3), and we observe that the capability of the Receiver greatly improved on both easy and hard tasks with the persuasion signal of the Advisor. In both subfigures, our method outperforms scenarios where no information or only prior information is given, and it even surpasses scenarios where all information is provided for most Receiver models.}
\label{fig:e2h_generalization}
\end{figure}

\subsubsection{Impact on Information Structure}\label{subsubsec:info_structure}
In our experiment, we also analyze the impact of different information structures on the Receiver. In the persuasion process, the receiver combines information items from the information set with input to generate a response. From the perspective of prompt engineering, we evaluate the quality of responses when the receiver either disregards information items or considers all information items, to demonstrate the effect of information selection. For `No Information', it serves as a baseline, equivalent to standard performance testing for LLMs. As shown in Table \ref{tab:math_results}, when the Receiver can access all information items, its performance improves. However, it is noteworthy that for some models, using all information items results in minimal gains or even a decline in performance compared to not using the information. It can be explained that providing too much information disperses the model's attention and risks exceeding the model's maximum window length. 

\subsubsection{Easy-to-Hard Generalization}\label{subsubsec:easy2hard}
In the extended evaluation, we investigate the generalizability of the Advisor. The results presented in Table \ref{tab:math_results} demonstrate that the Advisor's signaling strategy is effective across various Receiver, confirming its broad applicability. Following \citet{abs-2403-09472}, we evaluate our framework's \textit{Easy-to-Hard Generalization}, which is defined as the ability to address hard tasks by training on simpler ones. We train our Advisor on easy problems (levels 1-3) from the MATH training dataset and assess their effectiveness in persuading various models on both easy (levels 1-3) and complex problems (levels 4-5) of the MATH test dataset. As shown in Figure~\ref{fig:e2h_generalization}, we observe that advisors not only enhance the performance of various receiver models on easy problems but also improve their performance on hard problems, which are only trained exclusively with supervision on easy problems.

\begin{wrapfigure}{r}{0.5\textwidth}
\vspace{-2.0em}
    \centering
    \includegraphics[width=0.5\textwidth]{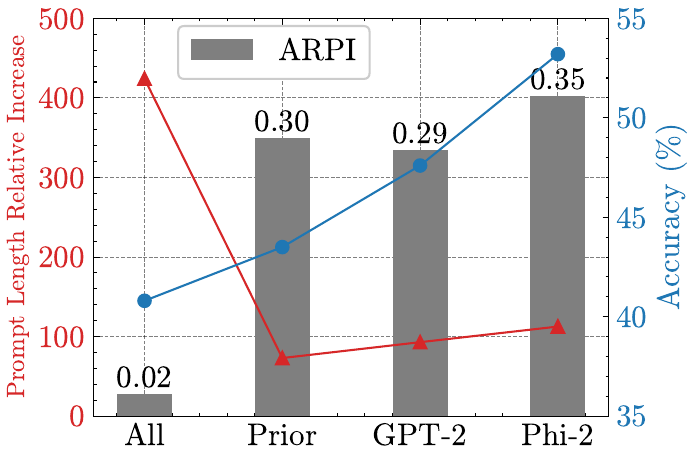}
    \vspace{-1.5em}
    \caption{Average Relative Accuracy Improvement on GSM8K. We compare two posterior structure from Advisor with `All Information' and `Prior Information'. The left y-axis represents the increase in prompt token length relative to the `No Information', while the right y-axis displays the average accuracy across various models on GSM8K.}
    \label{fig:efficiency}
\end{wrapfigure}

\subsubsection{Efficiency on Persuasion}\label{subsubsec:efficiency}
The efficiency of our framework lies in two aspects. On one hand, a well-trained Advisor can persuade various models to elicit better responses. On the other hand, during the inference stage, our method achieves enhanced Receiver's performance with fewer input tokens. To better understand the relationship between performance improvement and prompt length, we design the Average Relative Performance Improvement (ARPI) metric to measure the improvement in performance relative to the increase in prompt length. 

\textbf{Average Relative Performance Improvement (ARPI).} To compare the performance of a specific information structure with `No Information' across several receivers, let $R(A)$ represent the performance of structure $A$, and let $L$ denote the length of the input prompt tokens. We define $\textrm{ARPI}(A|B)$ as follows:
\begin{equation}
    \textrm{ARPI}(A|B) = \frac{1}{N} \sum_{i=1}^{N} \frac{R(A)-R(B)}{L(A)-L(B)}.
\end{equation}
$\textrm{ARPI}(A|B)$ presents the relative performance difference of structure $A$ compared with structure $B$. Figure~\ref{fig:efficiency} shows that when the Receiver uses all available information to generate responses, it improves performance relative to using no information, but this results in a 26\% increase in the length of input tokens, thereby increasing the computational cost of inference. In contrast, utilizing our persuasion framework, Phi-2 increases the input token length by only 6.9\% while achieving a 22.5\% performance improvement, leading to a higher efficiency ratio.

\vspace{-0.5em}
\section{Conclusion}
\vspace{-0.5em}
In this work, we introduce Bayesian Persuasion Alignment, a novel framework that integrates the concept of Bayesian persuasion with AI alignment. By formulating alignment as a Bayesian Persuasion problem, we employ a smaller model as an Advisor to generate signals that persuade a larger model, the Receiver, to enhance its performance. Our experimental results demonstrate significant improvements in the performance of various large models on mathematical problem-solving tasks and code generation tasks. The theoretical analysis provides an upper bound on the Advisor's regret, highlighting the efficacy of our method in learning the optimal signaling strategy. We hope our approach will inspire future research in integrating information design with alignment, contributing to the development of more efficient and effective AI systems.

\textbf{Limitations.}
The effectiveness of persuasion depends on the signaling strategy and is also influenced by the inherent capabilities of the Receiver. If the model itself lacks the ability to complete a certain task, our method may not be effective, which limits the applicability of our framework. 

\bibliographystyle{plainnat}
\bibliography{neurips2024}

\begin{thebibliography}{49}
\providecommand{\natexlab}[1]{#1}
\providecommand{\url}[1]{\texttt{#1}}
\expandafter\ifx\csname urlstyle\endcsname\relax
  \providecommand{\doi}[1]{doi: #1}\else
  \providecommand{\doi}{doi: \begingroup \urlstyle{rm}\Url}\fi

\bibitem[Abernethy et~al.(2008)Abernethy, Hazan, and
  Rakhlin]{abernethy2008competing}
Jacob~D Abernethy, Elad Hazan, and Alexander Rakhlin.
\newblock Competing in the dark: An efficient algorithm for bandit linear
  optimization.
\newblock In \emph{Conference on Learning Theory (COLT)}, pages 263--274.
  Citeseer, 2008.

\bibitem[Amodei et~al.(2016)Amodei, Olah, Steinhardt, Christiano, Schulman, and
  Man{\'e}]{amodei2016concrete}
Dario Amodei, Chris Olah, Jacob Steinhardt, Paul Christiano, John Schulman, and
  Dan Man{\'e}.
\newblock Concrete problems in ai safety.
\newblock \emph{arXiv preprint arXiv:1606.06565}, 2016.

\bibitem[Arieli and Babichenko(2019)]{arieli2019private}
Itai Arieli and Yakov Babichenko.
\newblock Private bayesian persuasion.
\newblock \emph{Journal of Economic Theory}, 182:\penalty0 185--217, 2019.

\bibitem[Austin et~al.(2021)Austin, Odena, Nye, Bosma, Michalewski, Dohan,
  Jiang, Cai, Terry, Le, and Sutton]{abs-2108-07732}
Jacob Austin, Augustus Odena, Maxwell~I. Nye, Maarten Bosma, Henryk
  Michalewski, David Dohan, Ellen Jiang, Carrie~J. Cai, Michael Terry, Quoc~V.
  Le, and Charles Sutton.
\newblock Program synthesis with large language models.
\newblock \emph{CoRR}, abs/2108.07732, 2021.

\bibitem[Bai et~al.(2022)Bai, Kadavath, Kundu, Askell, Kernion, Jones, Chen,
  Goldie, Mirhoseini, McKinnon, et~al.]{bai2022constitutional}
Yuntao Bai, Saurav Kadavath, Sandipan Kundu, Amanda Askell, Jackson Kernion,
  Andy Jones, Anna Chen, Anna Goldie, Azalia Mirhoseini, Cameron McKinnon,
  et~al.
\newblock Constitutional ai: Harmlessness from ai feedback.
\newblock \emph{arXiv preprint arXiv:2212.08073}, 2022.

\bibitem[Bernasconi et~al.(2023)Bernasconi, Castiglioni, Celli, Marchesi,
  Trov{\`o}, and Gatti]{bernasconi2023optimal}
Martino Bernasconi, Matteo Castiglioni, Andrea Celli, Alberto Marchesi,
  Francesco Trov{\`o}, and Nicola Gatti.
\newblock Optimal rates and efficient algorithms for online bayesian
  persuasion.
\newblock In \emph{International Conference on Machine Learning (ICML)}, pages
  2164--2183. PMLR, 2023.

\bibitem[Bowman et~al.(2022)Bowman, Hyun, Perez, Chen, Pettit, Heiner,
  Luko{\v{s}}i{\=u}t{\.e}, Askell, Jones, Chen, et~al.]{bowman2022measuring}
Samuel~R Bowman, Jeeyoon Hyun, Ethan Perez, Edwin Chen, Craig Pettit, Scott
  Heiner, Kamil{\.e} Luko{\v{s}}i{\=u}t{\.e}, Amanda Askell, Andy Jones, Anna
  Chen, et~al.
\newblock Measuring progress on scalable oversight for large language models.
\newblock \emph{arXiv preprint arXiv:2211.03540}, 2022.

\bibitem[Brown et~al.(2020)Brown, Mann, Ryder, Subbiah, Kaplan, Dhariwal,
  Neelakantan, Shyam, Sastry, Askell, et~al.]{brown2020language}
Tom Brown, Benjamin Mann, Nick Ryder, Melanie Subbiah, Jared~D Kaplan, Prafulla
  Dhariwal, Arvind Neelakantan, Pranav Shyam, Girish Sastry, Amanda Askell,
  et~al.
\newblock Language models are few-shot learners.
\newblock \emph{Advances in Neural Information Processing Systems (NeurIPS)},
  33:\penalty0 1877--1901, 2020.

\bibitem[Burns et~al.(2022)Burns, Ye, Klein, and
  Steinhardt]{burns2022discovering}
Collin Burns, Haotian Ye, Dan Klein, and Jacob Steinhardt.
\newblock Discovering latent knowledge in language models without supervision.
\newblock \emph{arXiv preprint arXiv:2212.03827}, 2022.

\bibitem[Burns et~al.(2023)Burns, Izmailov, Kirchner, Baker, Gao,
  Aschenbrenner, Chen, Ecoffet, Joglekar, Leike, et~al.]{burns2023weak}
Collin Burns, Pavel Izmailov, Jan~Hendrik Kirchner, Bowen Baker, Leo Gao,
  Leopold Aschenbrenner, Yining Chen, Adrien Ecoffet, Manas Joglekar, Jan
  Leike, et~al.
\newblock Weak-to-strong generalization: Eliciting strong capabilities with
  weak supervision.
\newblock \emph{arXiv preprint arXiv:2312.09390}, 2023.

\bibitem[Castiglioni et~al.(2020)Castiglioni, Celli, Marchesi, and
  Gatti]{castiglioni2020online}
Matteo Castiglioni, Andrea Celli, Alberto Marchesi, and Nicola Gatti.
\newblock Online bayesian persuasion.
\newblock \emph{Advances in Neural Information Processing Systems (NeurIPS)},
  33:\penalty0 16188--16198, 2020.

\bibitem[Chen et~al.(2021)Chen, Tworek, Jun, Yuan, de~Oliveira~Pinto, Kaplan,
  Edwards, Burda, Joseph, Brockman, Ray, Puri, Krueger, Petrov, Khlaaf, Sastry,
  Mishkin, Chan, Gray, Ryder, Pavlov, Power, Kaiser, Bavarian, Winter, Tillet,
  Such, Cummings, Plappert, Chantzis, Barnes, Herbert{-}Voss, Guss, Nichol,
  Paino, Tezak, Tang, Babuschkin, Balaji, Jain, Saunders, Hesse, Carr, Leike,
  Achiam, Misra, Morikawa, Radford, Knight, Brundage, Murati, Mayer, Welinder,
  McGrew, Amodei, McCandlish, Sutskever, and Zaremba]{abs-2107-03374}
Mark Chen, Jerry Tworek, Heewoo Jun, Qiming Yuan, Henrique~Pond{\'{e}}
  de~Oliveira~Pinto, Jared Kaplan, Harrison Edwards, Yuri Burda, Nicholas
  Joseph, Greg Brockman, Alex Ray, Raul Puri, Gretchen Krueger, Michael Petrov,
  Heidy Khlaaf, Girish Sastry, Pamela Mishkin, Brooke Chan, Scott Gray, Nick
  Ryder, Mikhail Pavlov, Alethea Power, Lukasz Kaiser, Mohammad Bavarian,
  Clemens Winter, Philippe Tillet, Felipe~Petroski Such, Dave Cummings,
  Matthias Plappert, Fotios Chantzis, Elizabeth Barnes, Ariel Herbert{-}Voss,
  William~Hebgen Guss, Alex Nichol, Alex Paino, Nikolas Tezak, Jie Tang, Igor
  Babuschkin, Suchir Balaji, Shantanu Jain, William Saunders, Christopher
  Hesse, Andrew~N. Carr, Jan Leike, Joshua Achiam, Vedant Misra, Evan Morikawa,
  Alec Radford, Matthew Knight, Miles Brundage, Mira Murati, Katie Mayer, Peter
  Welinder, Bob McGrew, Dario Amodei, Sam McCandlish, Ilya Sutskever, and
  Wojciech Zaremba.
\newblock Evaluating large language models trained on code.
\newblock \emph{CoRR}, abs/2107.03374, 2021.

\bibitem[Chiang et~al.(2023)Chiang, Li, Lin, Sheng, Wu, Zhang, Zheng, Zhuang,
  Zhuang, Gonzalez, Stoica, and Xing]{vicuna2023}
Wei-Lin Chiang, Zhuohan Li, Zi~Lin, Ying Sheng, Zhanghao Wu, Hao Zhang, Lianmin
  Zheng, Siyuan Zhuang, Yonghao Zhuang, Joseph~E. Gonzalez, Ion Stoica, and
  Eric~P. Xing.
\newblock Vicuna: An open-source chatbot impressing gpt-4 with 90\%* chatgpt
  quality, March 2023.
\newblock URL \url{https://lmsys.org/blog/2023-03-30-vicuna/}.

\bibitem[Christiano et~al.(2018)Christiano, Shlegeris, and
  Amodei]{christiano2018supervising}
Paul Christiano, Buck Shlegeris, and Dario Amodei.
\newblock Supervising strong learners by amplifying weak experts.
\newblock \emph{arXiv preprint arXiv:1810.08575}, 2018.

\bibitem[Christiano et~al.(2021)Christiano, Xu, and Cotra]{elk_intro}
Paul Christiano, Mark Xu, and Ajeya Cotra.
\newblock Arc's first technical report: Eliciting latent knowledge.
\newblock
  \url{https://www.alignmentforum.org/posts/qHCDysDnvhteW7kRd/arc-s-first-technical-report-eliciting-latent-knowledge},
  2021.

\bibitem[Christiano et~al.(2017)Christiano, Leike, Brown, Martic, Legg, and
  Amodei]{christiano2017deep}
Paul~F Christiano, Jan Leike, Tom Brown, Miljan Martic, Shane Legg, and Dario
  Amodei.
\newblock Deep reinforcement learning from human preferences.
\newblock \emph{Advances in Neural Information Processing Systems (NeurIPS)},
  30, 2017.

\bibitem[Cobbe et~al.(2021)Cobbe, Kosaraju, Bavarian, Chen, Jun, Kaiser,
  Plappert, Tworek, Hilton, Nakano, Hesse, and Schulman]{abs-2110-14168}
Karl Cobbe, Vineet Kosaraju, Mohammad Bavarian, Mark Chen, Heewoo Jun, Lukasz
  Kaiser, Matthias Plappert, Jerry Tworek, Jacob Hilton, Reiichiro Nakano,
  Christopher Hesse, and John Schulman.
\newblock Training verifiers to solve math word problems.
\newblock \emph{CoRR}, abs/2110.14168, 2021.

\bibitem[Cook and Webster(1972)]{cook1972caratheodory}
WD~Cook and RJ~Webster.
\newblock Caratheodory's theorem.
\newblock \emph{Canadian Mathematical Bulletin}, 15\penalty0 (2):\penalty0
  293--293, 1972.

\bibitem[Durmus et~al.(2024)Durmus, Lovitt, Tamkin, Ritchie, Clark, and
  Ganguli]{durmus2024persuasion}
Esin Durmus, Liane Lovitt, Alex Tamkin, Stuart Ritchie, Jack Clark, and Deep
  Ganguli.
\newblock Measuring the persuasiveness of language models, 2024.
\newblock URL
  \url{https://www.anthropic.com/news/measuring-model-persuasiveness}.

\bibitem[Evans et~al.(2021)Evans, Cotton-Barratt, Finnveden, Bales, Balwit,
  Wills, Righetti, and Saunders]{evans2021truthful}
Owain Evans, Owen Cotton-Barratt, Lukas Finnveden, Adam Bales, Avital Balwit,
  Peter Wills, Luca Righetti, and William Saunders.
\newblock Truthful ai: Developing and governing ai that does not lie.
\newblock \emph{arXiv preprint arXiv:2110.06674}, 2021.

\bibitem[Fogg(2002)]{fogg2002persuasive}
Brian~J Fogg.
\newblock Persuasive technology: using computers to change what we think and
  do.
\newblock \emph{Ubiquity}, 2002\penalty0 (December):\penalty0 2, 2002.

\bibitem[Gan et~al.(2022)Gan, Majumdar, Radanovic, and Singla]{gan2022bayesian}
Jiarui Gan, Rupak Majumdar, Goran Radanovic, and Adish Singla.
\newblock Bayesian persuasion in sequential decision-making.
\newblock In \emph{Proceedings of the AAAI Conference on Artificial
  Intelligence (AAAI)}, volume~36, pages 5025--5033, 2022.

\bibitem[Hendrycks et~al.(2021)Hendrycks, Burns, Kadavath, Arora, Basart, Tang,
  Song, and Steinhardt]{NEURIPS_BENCHMARKS2021_be83ab3e}
Dan Hendrycks, Collin Burns, Saurav Kadavath, Akul Arora, Steven Basart, Eric
  Tang, Dawn Song, and Jacob Steinhardt.
\newblock Measuring mathematical problem solving with the math dataset.
\newblock In \emph{Neural Information Processing Systems Track on Datasets and
  Benchmarks}, volume~1, 2021.

\bibitem[Hobbhahn(2022)]{elk_summary}
Marius Hobbhahn.
\newblock Eliciting latent knowledge (elk) - distillation/summary.
\newblock
  \url{https://www.alignmentforum.org/posts/rxoBY9CMkqDsHt25t/eliciting-latent-knowledge-elk-distillation-summary},
  2022.

\bibitem[Javaheripi et~al.(2023)Javaheripi, Bubeck, Abdin, Aneja, Bubeck,
  Mendes, Chen, Del~Giorno, Eldan, Gopi, et~al.]{javaheripi2023phi}
Mojan Javaheripi, S{\'e}bastien Bubeck, Marah Abdin, Jyoti Aneja, Sebastien
  Bubeck, Caio C{\'e}sar~Teodoro Mendes, Weizhu Chen, Allie Del~Giorno, Ronen
  Eldan, Sivakanth Gopi, et~al.
\newblock Phi-2: The surprising power of small language models.
\newblock \emph{Microsoft Research Blog}, 2023.

\bibitem[Ji et~al.(2023)Ji, Qiu, Chen, Zhang, Lou, Wang, Duan, He, Zhou, Zhang,
  et~al.]{ji2023ai}
Jiaming Ji, Tianyi Qiu, Boyuan Chen, Borong Zhang, Hantao Lou, Kaile Wang,
  Yawen Duan, Zhonghao He, Jiayi Zhou, Zhaowei Zhang, et~al.
\newblock Ai alignment: A comprehensive survey.
\newblock \emph{arXiv preprint arXiv:2310.19852}, 2023.

\bibitem[Ji et~al.(2024)Ji, Chen, Lou, Hong, Zhang, Pan, Dai, and
  Yang]{ji2024aligner}
Jiaming Ji, Boyuan Chen, Hantao Lou, Donghai Hong, Borong Zhang, Xuehai Pan,
  Juntao Dai, and Yaodong Yang.
\newblock Aligner: Achieving efficient alignment through weak-to-strong
  correction.
\newblock \emph{arXiv preprint arXiv:2402.02416}, 2024.

\bibitem[Jiang et~al.(2023)Jiang, Sablayrolles, Mensch, Bamford, Chaplot,
  Casas, Bressand, Lengyel, Lample, Saulnier, et~al.]{jiang2023mistral}
Albert~Q Jiang, Alexandre Sablayrolles, Arthur Mensch, Chris Bamford,
  Devendra~Singh Chaplot, Diego de~las Casas, Florian Bressand, Gianna Lengyel,
  Guillaume Lample, Lucile Saulnier, et~al.
\newblock Mistral 7b.
\newblock \emph{arXiv preprint arXiv:2310.06825}, 2023.

\bibitem[Kamenica(2019)]{kamenica2019bayesian}
Emir Kamenica.
\newblock Bayesian persuasion and information design.
\newblock \emph{Annual Review of Economics}, 11:\penalty0 249--272, 2019.

\bibitem[Kamenica and Gentzkow(2011)]{kamenica2011bayesian}
Emir Kamenica and Matthew Gentzkow.
\newblock Bayesian persuasion.
\newblock \emph{American Economic Review}, 101\penalty0 (6):\penalty0
  2590--2615, 2011.

\bibitem[Lee et~al.(2023)Lee, Phatale, Mansoor, Lu, Mesnard, Bishop, Carbune,
  and Rastogi]{lee2023rlaif}
Harrison Lee, Samrat Phatale, Hassan Mansoor, Kellie Lu, Thomas Mesnard, Colton
  Bishop, Victor Carbune, and Abhinav Rastogi.
\newblock Rlaif: Scaling reinforcement learning from human feedback with ai
  feedback.
\newblock \emph{arXiv preprint arXiv:2309.00267}, 2023.

\bibitem[Liu and Alahi(2024)]{liu2024co}
Yuejiang Liu and Alexandre Alahi.
\newblock Co-supervised learning: Improving weak-to-strong generalization with
  hierarchical mixture of experts.
\newblock \emph{arXiv preprint arXiv:2402.15505}, 2024.

\bibitem[Matz et~al.(2024)Matz, Teeny, Vaid, Peters, Harari, and
  Cerf]{matz2024potential}
SC~Matz, JD~Teeny, Sumer~S Vaid, H~Peters, GM~Harari, and M~Cerf.
\newblock The potential of generative ai for personalized persuasion at scale.
\newblock \emph{Scientific Reports}, 14\penalty0 (1):\penalty0 4692, 2024.

\bibitem[Nesterov and Nemirovskii(1994)]{nesterov1994interior}
Yurii Nesterov and Arkadii Nemirovskii.
\newblock \emph{Interior-point polynomial algorithms in convex programming}.
\newblock SIAM, 1994.

\bibitem[OpenAI(2023{\natexlab{a}})]{openai2023gpt4}
OpenAI.
\newblock Gpt-4 technical report, 2023{\natexlab{a}}.

\bibitem[OpenAI(2023{\natexlab{b}})]{superalignment}
OpenAI.
\newblock Introducing superalignment.
\newblock \url{https://openai.com/blog/introducing-superalignment},
  2023{\natexlab{b}}.
\newblock Accessed on July 5, 2023.

\bibitem[Ouyang et~al.(2022)Ouyang, Wu, Jiang, Almeida, Wainwright, Mishkin,
  Zhang, Agarwal, Slama, Ray, et~al.]{ouyang2022training}
Long Ouyang, Jeffrey Wu, Xu~Jiang, Diogo Almeida, Carroll Wainwright, Pamela
  Mishkin, Chong Zhang, Sandhini Agarwal, Katarina Slama, Alex Ray, et~al.
\newblock Training language models to follow instructions with human feedback.
\newblock \emph{Advances in Neural Information Processing Systems (NeurIPS)},
  35:\penalty0 27730--27744, 2022.

\bibitem[Pacchiardi et~al.(2023)Pacchiardi, Chan, Mindermann, Moscovitz, Pan,
  Gal, Evans, and Brauner]{pacchiardi2023catch}
Lorenzo Pacchiardi, Alex~J Chan, S{\"o}ren Mindermann, Ilan Moscovitz, Alexa~Y
  Pan, Yarin Gal, Owain Evans, and Jan Brauner.
\newblock How to catch an ai liar: Lie detection in black-box llms by asking
  unrelated questions.
\newblock \emph{arXiv preprint arXiv:2309.15840}, 2023.

\bibitem[Radford et~al.(2019)Radford, Wu, Child, Luan, Amodei, Sutskever,
  et~al.]{radford2019language}
Alec Radford, Jeffrey Wu, Rewon Child, David Luan, Dario Amodei, Ilya
  Sutskever, et~al.
\newblock Language models are unsupervised multitask learners.
\newblock \emph{OpenAI blog}, 1\penalty0 (8):\penalty0 9, 2019.

\bibitem[Rafailov et~al.(2023)Rafailov, Sharma, Mitchell, Manning, Ermon, and
  Finn]{rafailov2023direct}
Rafael Rafailov, Archit Sharma, Eric Mitchell, Christopher~D Manning, Stefano
  Ermon, and Chelsea Finn.
\newblock Direct preference optimization: Your language model is secretly a
  reward model.
\newblock In \emph{Advances in Neural Information Processing Systems
  (NeurIPS)}, 2023.
\newblock URL \url{https://openreview.net/forum?id=HPuSIXJaa9}.

\bibitem[Roziere et~al.(2023)Roziere, Gehring, Gloeckle, Sootla, Gat, Tan, Adi,
  Liu, Remez, Rapin, et~al.]{roziere2023code}
Baptiste Roziere, Jonas Gehring, Fabian Gloeckle, Sten Sootla, Itai Gat,
  Xiaoqing~Ellen Tan, Yossi Adi, Jingyu Liu, Tal Remez, J{\'e}r{\'e}my Rapin,
  et~al.
\newblock Code llama: Open foundation models for code.
\newblock \emph{arXiv preprint arXiv:2308.12950}, 2023.

\bibitem[Saunders et~al.(2022)Saunders, Yeh, Wu, Bills, Ouyang, Ward, and
  Leike]{saunders2022self}
William Saunders, Catherine Yeh, Jeff Wu, Steven Bills, Long Ouyang, Jonathan
  Ward, and Jan Leike.
\newblock Self-critiquing models for assisting human evaluators.
\newblock \emph{arXiv preprint arXiv:2206.05802}, 2022.

\bibitem[Sharma et~al.(2023)Sharma, Tong, Korbak, Duvenaud, Askell, Bowman,
  Cheng, Durmus, Hatfield-Dodds, Johnston, et~al.]{sharma2023towards}
Mrinank Sharma, Meg Tong, Tomasz Korbak, David Duvenaud, Amanda Askell,
  Samuel~R Bowman, Newton Cheng, Esin Durmus, Zac Hatfield-Dodds, Scott~R
  Johnston, et~al.
\newblock Towards understanding sycophancy in language models.
\newblock \emph{arXiv preprint arXiv:2310.13548}, 2023.

\bibitem[Sun et~al.(2024)Sun, Yu, Shen, Liu, Yang, Welleck, and
  Gan]{abs-2403-09472}
Zhiqing Sun, Longhui Yu, Yikang Shen, Weiyang Liu, Yiming Yang, Sean Welleck,
  and Chuang Gan.
\newblock Easy-to-hard generalization: Scalable alignment beyond human
  supervision.
\newblock \emph{CoRR}, abs/2403.09472, 2024.

\bibitem[Taori et~al.(2023)Taori, Gulrajani, Zhang, Dubois, Li, Guestrin,
  Liang, and Hashimoto]{taori2023stanford}
Rohan Taori, Ishaan Gulrajani, Tianyi Zhang, Yann Dubois, Xuechen Li, Carlos
  Guestrin, Percy Liang, and Tatsunori~B Hashimoto.
\newblock Stanford alpaca: An instruction-following llama model, 2023.

\bibitem[Touvron et~al.(2023)Touvron, Martin, Stone, Albert, Almahairi, Babaei,
  Bashlykov, Batra, Bhargava, Bhosale, et~al.]{touvron2023llama}
Hugo Touvron, Louis Martin, Kevin Stone, Peter Albert, Amjad Almahairi, Yasmine
  Babaei, Nikolay Bashlykov, Soumya Batra, Prajjwal Bhargava, Shruti Bhosale,
  et~al.
\newblock Llama 2: Open foundation and fine-tuned chat models.
\newblock \emph{arXiv preprint arXiv:2307.09288}, 2023.

\bibitem[Wang et~al.(2019)Wang, Shi, Kim, Oh, Yang, Zhang, and
  Yu]{wang2019persuasion}
Xuewei Wang, Weiyan Shi, Richard Kim, Yoojung Oh, Sijia Yang, Jingwen Zhang,
  and Zhou Yu.
\newblock Persuasion for good: Towards a personalized persuasive dialogue
  system for social good.
\newblock \emph{arXiv preprint arXiv:1906.06725}, 2019.

\bibitem[Zeng et~al.(2024)Zeng, Lin, Zhang, Yang, Jia, and Shi]{zeng2024johnny}
Yi~Zeng, Hongpeng Lin, Jingwen Zhang, Diyi Yang, Ruoxi Jia, and Weiyan Shi.
\newblock How johnny can persuade llms to jailbreak them: Rethinking persuasion
  to challenge ai safety by humanizing llms.
\newblock \emph{arXiv preprint arXiv:2401.06373}, 2024.

\bibitem[Zhang et~al.(2024)Zhang, Bai, Wang, Ye, Ma, and
  Yang]{zhang2024incentive}
Zhaowei Zhang, Fengshuo Bai, Mingzhi Wang, Haoyang Ye, Chengdong Ma, and
  Yaodong Yang.
\newblock Incentive compatibility for ai alignment in sociotechnical systems:
  Positions and prospects.
\newblock \emph{arXiv preprint arXiv:2402.12907}, 2024.

\end{thebibliography}

\definecolor{codegreen}{rgb}{0,0.6,0}
\definecolor{codegray}{rgb}{0.5,0.5,0.5}
\definecolor{codepurple}{rgb}{0.58,0,0.82}
\definecolor{backcolour}{rgb}{0.95,0.95,0.92}
\newcommand{\DM}[1]{{\color{blue} [DM: #1]}}
\newcommand{\Kuba}[1]{{\color{purple} [Kuba: #1]}}

\renewcommand*{\thefootnote}{(\arabic{footnote})}

\lstdefinestyle{mystyle}{
    backgroundcolor=\color{backcolour},   
    commentstyle=\color{codegreen},
    keywordstyle=\color{magenta},
    numberstyle=\tiny\color{codegray},
    stringstyle=\color{codepurple},
    basicstyle=\ttfamily\footnotesize,
    breakatwhitespace=false,         
    breaklines=true,                 
    captionpos=b,                
    keepspaces=true,                 
    numbers=left,                    
    numbersep=5pt,                  
    showspaces=false,                
    showstringspaces=false,
    showtabs=false,                  
    tabsize=2
}

\lstset{style=mystyle}

\hypersetup{
    colorlinks=true,
    linkcolor=blue,
    citecolor=blue,
    filecolor=magenta,      
    urlcolor=blue,
    linktocpage
}

\newpage
\appendix

\section{Experimental Details}\label{app:train_detail}
In this section, we detail the implementation, including the methodology for constructing information sets, the basic settings for training models, and the prompts utilized in the experiments.

\subsection{Information Set Construction}\label{app:prior_construction}
As the dataset itself does not contain information sets, we employ the open-source model Llama3-8B-Instruct \footnote{https://huggingface.co/meta-llama/Meta-Llama-3-8B-Instruct} to rapidly construct the information set for each question within all dataset. Specifically, during model inference, we set the temperature to 0.7 and the top-p to 0.9. Specifically, for mathematical problem-solving tasks, we generate information across seven aspects: Known Data, Objective, Methodology, Procedure Summary, Solution Verification, Assumptions, and Error Analysis. For code generation tasks, the generated aspects include Input and Output, Problem Statement, Test Cases, Logical Deduction, Algorithm Complexity, Error Handling, Edge Cases, and Code Structure. The entire process of information set construction is automated, with the decision on which aspects to generate also being determined by the LLM. For the prior distribution of information items, we calculate the conditional probabilities for each information item given a problem, generated by the model, and normalize these probabilities to form a valid distribution. We provide the complete prompts used for data generation in Appendix \ref{app:prompts}.

\begin{figure}[htbp]
    \centering
    \begin{graybox}
    A specific example 
    \end{graybox}
    \begin{lightbluebox}
    $<$\textbf{QUESTION}$>$\\
    \newtt{On Thursday the Meat Market sold 210kg of ground beef. On Friday they sold twice that amount. On Saturday they only sold 130kg. On Sunday they sold half of what they sold Saturday. If they originally planned to sell only 500kg, how much meat did they sell beyond their original plans?}\\
    
    $<$\textbf{EXTRA INFORMATION}$>$\\
    \newtt{
    "Known Data": Thursday: 210kg of ground beef sold; Friday: Twice the amount sold on Thursday, which is 2 x 210kg; Saturday: 130kg of ground beef sold; Sunday: Half of what was sold on Saturday, which is 0.5 x 130kg; Original plan: 500kg\\
    "Objective": To calculate how much meat was sold beyond the original plan\\
    "Methodology": 1. Calculate the total amount of ground beef sold on Thursday, Friday, Saturday, and Sunday. 2. Calculate the total amount of ground beef sold beyond the original plan\\
    "Procedure Summary": 1. Add the amount sold on Thursday, Friday, Saturday, and Sunday: 210 + 420 + 130 + 65 2. Subtract the original plan from the total amount sold\\
    "Solution Verification": 1. Check if the total amount sold is equal to the sum of the amounts sold on each day. 2. Check if the answer  makes sense in the context of the problem\\
    "Assumptions": The data provided is accurate and complete. The calculations are performed correctly\\
    "Error Analysis": Potential errors may occur due to incorrect calculation or misinterpretation of the data. Double-checking the calculations and verifying the answer against the given data can help identify and correct any errors\\
    }
    \end{lightbluebox}
\end{figure}

\subsection{Utility Function}\label{app:utility_func}
In our persuasion, the Receiver's utility function $u(x, c, y)$ is continuous and dependent on its response $y \in {\mathcal{Y}}$ to the input $x \in {\mathcal{X}}$ and the associated information item $c \in {\mathcal{C}}_x$. Similarly, the Advisor has a continuous utility function $v(x, c, y)$, which is contingent on the Receiver's response, input, and associated information item. In practical implementation, the utility function of the Advisor is defined as the logarithm of the probability of generating the correct answer, given the input $x$ and the information item $c$. For the utility function $u(x, c, y)$, a natural idea is to set the conditional probability $P(y|x,c)$ as utility. Autoregressive language model generate responses by continuously choosing the next token with the highest probability, ultimately producing the response with the maximum probability. This behavior aligns precisely with the Receiver selecting the optimal response based on the maximum utility. Therefore, in the experiment, given the input $x$ and information item $c$, the response generated by the Receiver is equivalent to selecting a response from the set (\ref{eq: receiver_op_y}).

\subsection{Training Hyperparameters}\label{app:train_hyper}
In our experiments, we train two models \textbf{GPT-2} \citep{radford2019language} (124M) and \textbf{Phi-2} \citep{javaheripi2023phi} (2.7B), utilizing the training datasets from GSM8K and MATH for mathematical problem-solving tasks, and MBPP for code generation tasks. Throughout the training, we employe the \texttt{AdamW} optimizer with hyperparameters set to $\beta_1 = 0.9$, $\beta_2 = 0.95$, and $\epsilon = 10^{-5}$. Additionally, We use a cosine learning rate schedule with a maximum learning rate of $5 \times 10^{-5}$. All models are trained on 4 NVIDIA A800 GPUs.

\subsection{Prompts}\label{app:prompts}
In our experiment, we employ distinct prompts at different stages, which included the construction of the information set for math problems and code generation tasks, the generation of signals by the Advisor, and the generation of responses by the Receiver. Here, we present the prompts used throughout our experiment.

\subsubsection{Math Tasks}
\begin{figure}[htbp]
    \centering
    \begin{graybox}
    Prompt for the construction of information sets on math problems
    \end{graybox}
    \begin{lightbluebox}
    $<$\textbf{INSTRUCTION}$>$\\
    Please provide key information on the following aspects:\\
    
    1. Known Data: List all numerical values and conditions given in the problem.\\
    2. Objective: Clearly define the specific calculation or problem that needs to be solved.\\
    3. Methodology: Describe the mathematical formulas or logical reasoning required to solve the problem.\\
    4. Procedure Summary: Outline the solution steps from the given data to the resolution of the problem.\\
    5. Solution Verification: Suggest methods to verify the correctness of the answer.\\
    6. Assumptions: List any assumptions made to simplify the problem or calculation.\\
    7. Error Analysis: Identify potential errors or mistakes that may occur during the calculation.\\
    
    Ensure that the information provided is accurate, precise to facilitate the correct solution.\\
    
    $<$\textbf{QUESTION}$>$\\
    ...
    \end{lightbluebox}
\end{figure}

\subsubsection{Response Generation}
\begin{figure}[htbp]
    \centering
    \begin{graybox}
    Prompt for Receiver generate a response
    \end{graybox}
    \begin{lightbluebox}
    $<$\textbf{QUESTION}$>$\\
    ...\\
    $<$\textbf{EXTRA INFORMATION}$>$\\
    Integrate with the additional context to form a thorough and insightful answer.\\
    \{\newtt{information item}\}\\
    
    $<$\textbf{ANSWER}$>$\\
    Let's think step by step.\\
    \end{lightbluebox}
\end{figure}

\newpage
\subsubsection{Code Generation Tasks}
\begin{figure}[h!]
    \centering
    \begin{graybox}
    Prompt for the construction of information sets on code task
    \end{graybox}
    \begin{lightbluebox}
    $<$\textbf{INSTRUCTION}$>$\\
    Please provide key information on the following aspects:\\
    
    1. Input and Output: Clearly specify the function's parameters and return types.\\
    2. Problem Statement: Understand the problem to be solved and the expected solution.\\  
    3. Test Cases: Design test cases based on edge cases and special situations.\\  
    4. Logical Deduction: Determine the basic logic for solving the problem based on the description and examples.\\
    5. Algorithm Complexity: Evaluate the time and space complexity of the designed algorithm.\\
    6. Error Handling: Consider handling potential errors and exceptions.\\
    7. Edge Cases: Identify extreme cases in the problem.\\
    
    Ensure that the information provided is accurate, precise to facilitate the correct solution.\\
    
    $<$\textbf{QUESTION}$>$\\
    ...
    \end{lightbluebox}
\end{figure}

\subsubsection{Signal Generation}
\begin{figure}[h]
    \centering
    \begin{graybox}
    Prompt for Advisor generate a signal
    \end{graybox}
    \begin{lightbluebox}
    $<$\textbf{INSTRUCTION}$>$\\
     Summarize below information and present the most important details in an accurate and precise format.\\
    $<$\textbf{EXTRA INFORMATION}$>$\\
    \{\newtt{all information items}\}\\
    \end{lightbluebox}
\end{figure}

\subsection{The Training Code of Persuasion}
The pseudocode below shows the basic training process of our Bayesian persuasion framework. 
\begin{lstlisting}[language=Python]
def compute_loss(advisor, inputs):
    # Step1. smaple a information item
    sample_infos = ...
    # Step2. produce signals by Advisor, condition on sample_infos
    signals_ids = advisor.generate(max_new_tokens=50,)
    # Step3. update belief on info by bayes rule
    posterior_belief = update_posterior_belief(prior, info_items, signals_ids)
    # Step4. Receiver choose best response
    item_index = torch.max(posterior_belief, dim=1)
    ...
    advisor_utility = get_advisor_utility(receiver, inputs_ids, sample_infos)
    # Step5. calculate the loss of Advisor
    loss = (posterior_belief * advisor_utility).mean()
    return loss
\end{lstlisting}

\section{Additional Experiments}

\subsection{Generalization on various tasks}\label{app:general_on_task}
{\large
\setlength{\extrarowheight}{1.1pt}
\begin{table}[htbp]
\centering
\caption{\textbf{Performance of various Receivers under persuasion.} 
We report the accuracy on GSM8K and MATH. "Posterior Information" refers to sampling the information item from the posterior distribution, influenced by the Advisor. The Advisor for math tasks differs from that for code generation tasks. Arrows indicate performance improvements relative to the prior distribution.}

\label{tab:various_task_results}
\resizebox{\textwidth}{!}{
\begin{tabular}{@{}clccccc@{}}
\toprule
\multirow{2}{*}{Task} &
  \multicolumn{1}{l}{\multirow{2}{*}{Receiver}} &
  \multicolumn{1}{c}{\multirow{2}{*}{\parbox{2cm}{\centering No \\Information}}} &
  \multicolumn{1}{c}{\multirow{2}{*}{\parbox{2cm}{\centering All \\Information}}} &
  \multicolumn{1}{c}{\multirow{2}{*}{\parbox{2cm}{\centering Prior \\Information}}} &
  \multicolumn{2}{c}{\textbf{Posterior Information}} \\ \cmidrule(l){6-7}\cmidrule(l){6-7}
 &
  \multicolumn{1}{c}{} &
  \multicolumn{1}{c}{} &
  \multicolumn{1}{c}{} &
  \multicolumn{1}{c}{} &
  \multicolumn{1}{c}{Advisor (GPT-2)} &
  \multicolumn{1}{c}{Advisor (Phi-2)} \\ \midrule
\multirow{12}{*}{\parbox{2cm}{\centering GSM8K \\ (8-shot)}}
        & Phi-2      &56.0  &41.0  &56.8  &57.3  &59.3  \\
        & Mistral-7B &34.3  &48.0  &45.7  &47.3  &51.3  \\
        & Llama2-7B  &15.1  &36.6  &27.2  &33.0  &44.0  \\
    & Llama2-7B-Chat &21.8  &31.8  &37.3  &40.3  &48.3  \\
        & Llama2-13B &25.2  &38.9  &36.2  &41.7  &43.7  \\
    & Llama2-13B-Chat&33.9  &37.3  &36.1  &37.7  &37.7  \\
        & Llama3-8B  &47.6  &54.0  &53.7  &54.4  &54.4  \\
&Llama3-8B-Instruct  &73.5  &72.2  &72.3  &74.3  &74.3  \\
        & Vicuna-7B  &14.9  &19.9  &29.9  &32.6  &39.6  \\
        & Vicuna-13B &23.0  &24.8  &35.0  &38.3  &45.3  \\
        & Vicuna-33B &43.2  &44.1  &47.8  &50.8  &55.8  \\ \cdashline{2-7}\cdashline{2-7}
& Average (accuracy) &35.3  &40.8  &43.5  &46.2 (\textbf{6.2\%} $\uparrow$) &50.3  (\textbf{15.6\%} $\uparrow$) \\ \midrule
\multirow{12}{*}{\parbox{2cm}{\centering MATH \\ (4-shot)}}
        & Phi-2      &10.1  &11.6  &11.5  &13.8  &14.8  \\
        & Mistral-7B &6.4   &10.3  &7.9   &9.1   &10.8  \\
        & Llama2-7B  &4.1   &9.5   &6.3   &8.5   &10.3  \\
    & Llama2-7B-Chat &4.6   &7.8   &6.0   &7.9   &9.3  \\
        & Llama2-13B &4.5   &9.7   &7.7   &9.3   &10.5  \\
    & Llama2-13B-Chat&5.2   &9.8   &7.3   &8.8   &10.4  \\
        & Llama3-8B  &11.0  &16.1  &12.8  &15.5  &16.0  \\
&Llama3-8B-Instruct  &18.1  &18.6  &18.1  &18.8  &19.3  \\
        & Vicuna-7B  &3.8   &10.1  &6.7   &8.7   &11.1  \\
        & Vicuna-13B &3.8   &11.1  &6.7   &9.1   &13.0  \\
        & Vicuna-33B &6.8   &13.1  &9.3   &10.6  &12.2  \\ \cdashline{2-7}\cdashline{2-7}
& Average (accuracy) &7.1   &11.6  &9.1   &10.9  (\textbf{19.8\%} $\uparrow$) &12.5  (\textbf{37.3\%} $\uparrow$) \\ \midrule
\end{tabular}
}
\end{table}
}

We investigate the generalization of the signaling strategy across different Receivers in Section~\ref{subsubsec:perf_persuasion} and across varying levels of difficulty in Section~\ref{subsubsec:easy2hard}. In this section, we evaluate the generalization of our well-trained signaling strategy across various tasks. Specifically, we train the Advisor using the GSM8K training dataset and assess its performance on the MATH dataset. Concurrently, 
we train the Advisor using the MATH training dataset and assess its performance on GSM8K. As shown in Table \ref{tab:various_task_results}, the Advisor is capable of enhancing the performance of all Receivers on both GSM8K and MATH to varying degrees. When GPT-2 acts as the Advisor, it facilitates performance improvements for multiple Receivers, with an average performance increase of 6.2\% on GSM8K and 19.8\% on MATH. In contrast, Phi-2 achieves more notable performance enhancements, with gains of 15.6\% on GSM8K and 37.3\% on MATH.

\section{Proof of Theorem \ref{algo:online_bp}}
\label{appendix:proof_theorem}

\begin{assumption}
The prior distribution $\mu^{0}$ is in the interior of $\Delta({\mathcal{C}})$, i.e., $\mu_{c}^{0} > 0$ for all $c \in {\mathcal{C}}$.
\end{assumption}

\begin{definition}[Linear Map]
A vector-valued function $f: X \to \mathbb{R}^D$ is said to be linear if there exists a matrix $M \in \mathbb{R}^{D \times M}$ such that $f(x) = Mx$ for all $x \in X \subseteq \mathbb{R}^M$. 
\end{definition}

\begin{theorem}[Caratheodory's Theorem \citep{cook1972caratheodory}]
Let $S \subseteq \mathbb{R}^D$ be a set. Then, any point in the convex hull of $S$ can be expressed as a convex combination of at most $D+1$ points from $S$.
\end{theorem}

\begin{theorem}[\cite{bernasconi2023optimal}, Theorem 3.2]
\label{th:polytope}
If $X$ is a polytope and $f$ is a linear map, then there exist algorithms implementing the Carathéodory decomposition and the inverse map $f^{\dagger}$.
\end{theorem}

\begin{proof}
If $X$ is a polytope and $f$ is a linear map, then $f$ is also a polytope. By Caratheodory's theorem, every point in $f(X)$ can be expressed as a convex combination of at most $D+1$ vertices of $f(X)$, where $D$ is the dimension of $f(X)$.

Given any $z \in f(X)$, the Carathéodory decomposition algorithm finds at most $D+1$ vertices $\{z_1,\ldots,z_{D+1}\} \subseteq f(X)$ and convex coefficients $\{\lambda_1,\ldots,\lambda_{D+1}\}$ such that $z= \sum_{i=1}^{D+1} \lambda_i z_i$. This can be done by solving a linear program.

For the inverse map, given any $z \in f(X)$, we want to find an $x \in X$ such that $f(x) = z$. Since $f$ is linear, we have $f(x) = Mx$ for some matrix $M$. Finding $x$ is thus equivalent to solving the linear system $Mx = z$. Since $z \in f(X)$, this system is guaranteed to have a solution, which can be found using Gaussian elimination.
\end{proof}

\begin{corollary}[\cite{bernasconi2023optimal}, Corollary 3.4]
\label{cor:regret_bound}
Under the assumptions of Theorem \ref{th:polytope}, there exists a regret minimizer $\mathcal{R}$ such that Algorithm \ref{algo:online_bp} guarantees cumulative regret
\begin{equation}
    R_T \leq 16 D^{3/2} \sqrt{T \log T},
\end{equation}
where $D$ is the dimension of $f(X)$.
\end{corollary}

\begin{proof}
To obtain the regret bound, we equip Algorithm \ref{algo:online_bp} with a suitably-defined regret minimizer $\mathcal{R}$. In particular, $\mathcal{R}$ works by observing the realized utility $v(y_t, c_t)$, since the sender does not directly play $\phi_t$, but rather draws an action $y_t$ according to $\phi_{t,c_t}$. Such a regret minimizer $\mathcal{R}$ can be implemented by the algorithm introduced by \citep{abernethy2008competing}, as any polytope in $\mathbb{R}^D$ has a D-self concordant barrier \citep{nesterov1994interior} (Theorem 2.5.1). This yields the stated regret bound \cite{abernethy2008competing} (Theorem 1).
\end{proof}

With the above theorems and corollary, we are now ready to prove Thereom \ref{algo:online_bp}.
\begin{proof}
The proof of this theorem will proceed in three steps.

\textbf{Step 1:} Show that the set of direct and persuasive signaling schemes ${\mathcal{P}}$ is a polytope.

To see this, note that ${\mathcal{P}}$ can be described by the following linear constraints:
\begin{align}
    & \sum_{y \in {\mathcal{Y}}} \phi_{c}(y) = 1, \forall c \in {\mathcal{C}}, \label{eq:prob_dist}\\
    & \sum_{c \in {\mathcal{C}}} \mu_{c}\phi_{{\mathcal{C}}}(y)(v(y,c) - v(y',c)) \geq 0, \forall y, y'\in {\mathcal{Y}}, \label{eq:persuasive}\\
    & \phi_{c}(y) \geq 0, \forall c \in {\mathcal{C}}, y \in {\mathcal{Y}}. \label{eq:non-negative}
\end{align}
Constraint \eqref{eq:prob_dist} ensures that $\phi_{c}$ is a valid probability distribution for each $c$. Constraint \eqref{eq:persuasive} is the persuasiveness constraint. Constraint \eqref{eq:non-negative} ensures non-negativity. As these are all linear constraints, ${\mathcal{P}}$ is a polytope.

\textbf{Step 2:} Define a linear map $f: {\mathcal{P}} \to \mathbb{R}^m$.

Let $f: {\mathcal{Y}} \to \mathbb{R}^m$ be defined as $f(\phi) = [-v(\phi, y)]_{y \in {\mathcal{Y}}}$ for all $\phi \in {\mathcal{P}}$, where $v(\phi, y) = \sum_{c \in {\mathcal{C}}} \mu_{c} \phi_{c}(y) v(y, c)$ is the Advisor's expected utility for action $y$ under signaling scheme $\phi$. We can verify that $f$ is linear:
\begin{align*}
    f(\alpha \phi_1 + \beta \phi_2) &= [-v(\alpha \phi_1 + \beta \phi_2, y)]_{y \in {\mathcal{Y}}} \\
    &= [-\alpha v(\phi_1, y) - \beta v(\phi_2, y)]_{y \in {\mathcal{Y}}} \\
    &= \alpha [-v(\phi_1, y)]_{y \in {\mathcal{Y}}} + \beta [-v(\phi_2, y)]_{y \in {\mathcal{Y}}} \\
    &= \alpha f(\phi_1) + \beta f(\phi_2)
\end{align*}
for any $\phi_1, \phi_2 \in {\mathcal{P}}$ and $\alpha, \beta \in \mathbb{R}$.

\textbf{Step 3:} Apply Corollary \ref{cor:regret_bound} to derive the regret bound.

Since $f$ is a linear map from ${\mathcal{P}}$ to $\mathbb{R}^m$, we have $f({\mathcal{P}}) \subseteq \mathbb{R}^m$, so the dimension of $f({\mathcal{P}})$ is at most $m$. By Corollary \ref{cor:regret_bound}, if the dimension of $f({\mathcal{P}})$ is $D$, then there exists a regret minimizer with regret bound
\begin{equation*}
    R_T \leq 16 D^{3/2} \sqrt{T \log T}.
\end{equation*}
In our setting, $D \leq m$. Therefore, we get the following regret bound:
\begin{equation*}
    R_T = O(m^{3/2} \sqrt{T \log T}).
\end{equation*}

Putting everything together, we have shown that the regret of Algorithm \ref{algo:online_bp} is upper bounded by $O(m^{3/2} \sqrt{T \log T})$, where $m = |{\mathcal{Y}}|$ is the size of the Receiver's
response space.
\end{proof}


\end{document}